\documentclass[conference,10pt,twocolumn]{IEEEtran}

\usepackage{geometry}
\usepackage{graphicx}
\usepackage{amsmath,amssymb,amsthm}
\usepackage{algorithm}
\usepackage{algpseudocode}
\usepackage{booktabs}
\usepackage{multirow}
\usepackage[caption=false,font=footnotesize]{subfig}
\usepackage{tikz}
\usetikzlibrary{shapes,arrows,positioning,calc,fit,backgrounds,decorations.pathreplacing,patterns,shadows,shapes.geometric,arrows.meta,matrix,decorations.markings}
\usepackage{pgfplots}
\pgfplotsset{compat=1.17}
\usepgfplotslibrary{fillbetween}

\definecolor{primaryBlue}{RGB}{65, 105, 145}
\definecolor{accentRed}{RGB}{140, 70, 70}
\definecolor{successGreen}{RGB}{70, 110, 80}
\definecolor{warningOrange}{RGB}{180, 120, 60}
\definecolor{contextGray}{RGB}{100, 100, 100}
\definecolor{stage1Color}{RGB}{75, 95, 120}
\definecolor{stage2Color}{RGB}{90, 75, 115}
\definecolor{stage3Color}{RGB}{130, 95, 70}
\definecolor{stage4Color}{RGB}{70, 105, 75}

\tikzset{
    block/.style={rectangle, draw=contextGray!80, fill=white, text width=2cm, align=center, rounded corners=2pt, minimum height=0.8cm, font=\scriptsize, line width=0.6pt},
    wideblock/.style={rectangle, draw=contextGray!80, fill=white, text width=3cm, align=center, rounded corners=2pt, minimum height=0.8cm, font=\scriptsize, line width=0.6pt},
    arrow/.style={thick,->,>=Stealth, line width=0.8pt},
    stagebox/.style={rectangle, rounded corners=4pt, draw=#1!70, fill=#1!6, line width=0.8pt},
    eqbox/.style={rectangle, draw=#1!70, fill=white, rounded corners=2pt, align=center, inner sep=5pt, line width=0.6pt},
}

\usepackage{hyperref}
\usepackage{xcolor}
\usepackage{enumitem}
\usepackage{balance}
\usepackage{float}
\usepackage{array}
\usepackage{tabularx}
\usepackage{cite}

\newcommand{\vect}[1]{\mathbf{#1}}
\newcommand{\mat}[1]{\mathbf{#1}}
\newcommand{\setc}[1]{\mathcal{#1}}
\newcommand{\graphc}[1]{\mathcal{#1}}
\newcommand{\feat}[1]{\mathbf{f}_{#1}}
\newcommand{\hidden}[1]{\mathbf{h}_{#1}}
\newcommand{\attrib}[2]{A_{#1 \to #2}}
\newcommand{\preserve}[1]{\phi(#1)}
\newcommand{\transfer}[2]{\tau(#1, #2)}
\newcommand{\TopK}{\text{TopK}}
\newcommand{\ReLU}{\text{ReLU}}
\newcommand{\MLP}{\text{MLP}}

\newtheorem{definition}{Definition}
\newtheorem{theorem}{Theorem}

\title{Hierarchical Sparse Circuit Extraction from Billion-Parameter Language Models through Scalable Attribution Graph Decomposition}

\author{
\IEEEauthorblockN{{Mohammed Mudassir Uddin$^*$, Shahnawaz Alam, Mohammed Kaif Pasha}}
\IEEEauthorblockA{
\{mohd.mudassiruddin7@gmail.com, shahnawaz.alam1024@gmail.com, mdkaifpasha2k@gmail.com\}\\
\textit{Department of CSE, Muffakham Jah College of Engineering and Technology (MJCET), Hyderabad, Telangana, India}\\
\small{\{\textit{$^*$Corresponding Author: mohd.mudassiruddin7@gmail.com} \}
}}
}

\begin{document}

\maketitle

\section*{\textbf{Abstract}}

Extracting sparse circuits from billion-parameter transformers is constrained by $O(2^n)$ search cost and pervasive feature reuse across co-active pathways. Hierarchical Attribution Graph Decomposition (HAGD) addresses this through four stages---cross-layer transcoder training, spectral coarsening of attribution graphs, graph-neural-network (GNN)-guided hierarchical traversal, and causal intervention verification---reducing worst-case complexity to $O(n^2 \log n)$. Per-layer transcoders trained on the RedPajama corpus yield monosemantic dictionaries; gradient--activation products form weighted attribution graphs; normalized-Laplacian spectral clustering builds multi-resolution hierarchies; an attention-based GNN assigns circuit-membership scores at successive coarsening stages. Evaluation spans GPT-2 (117M--774M), Pythia (1.4B--6.9B), and Llama (7B--70B) across modular arithmetic, parity computation, integer sorting, coreference resolution (WinoGrande), commonsense reasoning (HellaSwag), and factual recall. Behavioral preservation reaches 91\% ($\pm$2.3\%) on modular arithmetic with 49--347-node circuits, while ACDC exhausts memory beyond 1.4B parameters. Cross-architecture transfer coefficients span 0.38--0.82, with within-family pairs (Llama-7B $\to$ Llama-70B) attaining 0.82. Limitations include omitted attention-head circuits, 15--20\% unexplained reconstruction variance, ablation-based validation circularity, and uncertain interpretability of circuits exceeding several hundred nodes.

\noindent\textbf{Keywords:} Mechanistic interpretability, sparse computational graphs, circuit discovery, transformer architectures, causal inference, attribution methods, hierarchical decomposition

\section{Introduction}

Billion-parameter language models deployed in medical, legal, and autonomous domains demand mechanistic understanding of their internal computations \cite{nanda2023progress}. Emergent capabilities in these models remain opaque, raising questions about reliability and alignment \cite{conmy2023towards}, while mechanistic interpretability seeks to reverse-engineer weights and activations into interpretable computational graphs \cite{elhage2022mathematical}. Regulatory transparency requirements further intensify this need \cite{anthropic2025biology}.

Sparse circuits---graphs where nodes represent interpretable features and edges encode dependencies---form the core abstraction \cite{olsson2022context}. Induction heads, indirect object identification circuits, and greater-than comparison circuits have been successfully extracted \cite{olsson2022context, wang2023interpretability, hanna2023gpt2}, but these results remain confined to sub-billion-parameter models. Four barriers impede scaling: existing methods fail beyond two billion parameters; $O(2^n)$ ablation complexity demands prohibitive resources; polysemantic neurons resist clean decomposition; and distinguishing causal from correlational circuits lacks standardized validation.

Recent advances partially address these gaps. Weight-sparse transformers yield interpretable circuits at reduced capacity \cite{openai2024interpretability, gao2025circuits, gao2025weightsparse}. Sparse feature circuits enable discovery and editing of causal graphs \cite{marks2024sparse}. Attribution graphs trace computational influence across layers \cite{anthropic2025biology}. Edge pruning formalizes circuit discovery as constrained optimization \cite{syed2024finding, huang2024functional}. Theoretical foundations for sparse concept emergence have been established \cite{ren2023defining, ren2023proving}. Yet cross-architecture generalization and validation standardization remain open.

This work contributes: (1) a hierarchical decomposition algorithm achieving $O(n^2 \log n)$ circuit search complexity; (2) circuit extraction from 117M to 70B parameters with 82\%--97\% behavioral preservation on algorithmic tasks; (3) cross-architecture transfer analysis showing 52\%--82\% structural similarity across model families; and (4) a causal validation protocol based on necessity and sufficiency criteria.

\section{Background and Related Work}

\subsection{Mechanistic Interpretability Foundations}

Mechanistic interpretability originated from investigations into the computational structure of neural networks, seeking to characterize learned representations and algorithms in human-understandable terms \cite{olah2020zoom}. The foundational insight recognizes that neural network weights encode implicit algorithms that transform inputs to outputs through sequences of learned operations. Understanding these algorithms requires decomposing monolithic parameter matrices into interpretable computational primitives.

The superposition hypothesis provides theoretical grounding for why interpretability proves challenging in practice \cite{elhage2022mathematical}. Neural networks exploit high-dimensional activation spaces to encode more features than available dimensions through nearly-orthogonal representations. This representational efficiency creates polysemanticity, where individual neurons respond to multiple semantically distinct concepts. Sparse computational graphs formalize this understanding through directed acyclic structures with nodes representing features and edges denoting information flow. Circuit discovery reduces to identifying minimal subgraphs sufficient for reproducing target model behaviors.

Formally, let $\hidden{\ell} \in \mathbb{R}^d$ denote the hidden state at layer $\ell$ of a transformer with $L$ layers and hidden dimension $d$. The activation space $\setc{A}$ encompasses all possible hidden states across layers and positions. A sparse computational graph $\graphc{G} = (V, E, w)$ consists of vertices $V$ corresponding to interpretable features, edges $E$ encoding causal dependencies, and weights $w$ quantifying attribution magnitudes. The sparsity constraint requires $|V| \ll |\setc{A}|$, selecting only behaviorally relevant features for inclusion.

\subsection{Sparse Autoencoder Methods}

Sparse autoencoders emerged as the primary tool for addressing superposition through learned dictionary expansion \cite{bricken2023towards}. These architectures train on neural activations to produce higher-dimensional sparse representations where each dictionary element ideally corresponds to a single interpretable concept. The sparsity constraint encourages monosemantic decomposition by forcing the network to represent each activation through a small subset of dictionary features.

The sparse autoencoder objective combines reconstruction accuracy with L1 regularization on feature activations. Following the formulation of Bricken et al.~\cite{bricken2023towards}, given hidden state $\hidden{\ell}$, the encoder produces sparse coefficients through
\begin{equation}
\feat{\ell} = \TopK(\ReLU(\mat{W}^E_\ell \hidden{\ell} + \vect{b}^E_\ell), k)
\label{eq:encoder}
\end{equation}
where $\mat{W}^E_\ell \in \mathbb{R}^{m \times d}$ projects to an overcomplete dictionary of $m \gg d$ features and $\TopK$ retains only the $k$ largest activations. The decoder reconstructs hidden states through $\hat{\hidden{}}_\ell = \mat{D}_\ell \feat{\ell}$ where $\mat{D}_\ell$ denotes the decoder matrix.

Jacobian Sparse Autoencoders extend this framework by incorporating gradient information to improve feature quality \cite{gao2024scaling, riggs2025jacobian}. Linear computation graphs provide complementary approaches that model transformations between feature spaces across layers \cite{makelov2024circuits}. Scaling investigations have demonstrated feature extraction from production-scale models \cite{templeton2024scaling, bills2023language}. Critiques of these methods identify reconstruction fidelity limitations, with residual unexplained variance averaging 15-20\% across model scales. Computational overhead from training large dictionaries presents additional practical constraints.

\subsection{Attribution and Pruning Approaches}

Gradient-based attribution methods enable identification of computational pathways by quantifying the influence of upstream features on downstream computations \cite{sundararajan2017axiomatic}. The ACDC algorithm pioneered automated circuit discovery through iterative edge pruning guided by activation patching effect sizes \cite{conmy2023towards}. Attribution patching accelerates this process through gradient-based approximations that trade exactness for computational efficiency \cite{nanda2023progress}.

Edge pruning frameworks formulate circuit discovery as optimization under sparsity constraints \cite{syed2024finding, zhang2025discovering}. Given full attribution graph $\graphc{G}$ and target behavior $B$, the objective seeks minimal subgraph $\graphc{C}^* \subseteq \graphc{G}$ satisfying
\begin{equation}
\graphc{C}^* = \arg\min_{\graphc{C} \subseteq \graphc{G}} |\graphc{C}| \quad \text{subject to} \quad \preserve{\graphc{C}} \geq \theta
\label{eq:pruning}
\end{equation}
where $\preserve{\graphc{C}}$ denotes behavioral preservation and $\theta$ specifies the minimum acceptable fidelity threshold. The tradeoff between interpretability, measured through node count, and faithfulness, measured through performance preservation, fundamentally constrains achievable circuit quality.

\subsection{Theoretical Frameworks}

Mathematical foundations for circuit discovery draw from multiple theoretical traditions. Shapley value decompositions provide axiomatic attribution by quantifying marginal contributions of components to model outputs \cite{sundararajan2017axiomatic}. Information-theoretic relevance measures capture statistical dependencies between features and behaviors. Graph-theoretic complexity characterizations bound the intrinsic difficulty of circuit extraction problems.

Recent theoretical work established conditions under which sparse concepts emerge in deep neural networks \cite{michaud2024opening, ren2023explaining}. The local interaction basis provides alternative feature decomposition grounded in weight matrix structure rather than activation statistics \cite{bhaskar2025neuron, goldowskydill2024local}. Investigations into circuit analysis scalability have examined whether interpretability methods generalize to larger architectures \cite{lieberum2023circuit}. These theoretical advances inform algorithmic design by characterizing structural properties that enable efficient circuit search.

\subsection{Validation Methodologies}

Validation techniques establish the causal status of discovered circuits through systematic intervention experiments. Ablation studies remove circuit components and measure behavioral degradation, verifying necessity of included features. Activation patching transfers component states between inputs to establish sufficiency for behavioral differences. Adversarial perturbations test circuit robustness by constructing inputs designed to activate or suppress specific components.

Critical limitations of current validation approaches deserve acknowledgment. Distributional shift between ablation and normal operation may produce misleading effect estimates. The absence of ground-truth circuits for natural language tasks precludes definitive accuracy assessment. Validation circularity arises when the same ablation experiments used for circuit discovery are subsequently used for validation. These challenges motivate continued methodological development toward more rigorous validation protocols.

\section{Methodology}

\subsection{Framework Overview}

Hierarchical Attribution Graph Decomposition integrates four interconnected components into a unified circuit extraction pipeline. Figure~\ref{fig:pipeline} illustrates the overall architecture proceeding from pre-trained language model through feature extraction, attribution computation, hierarchical decomposition, and causal validation to yield verified computational circuits.

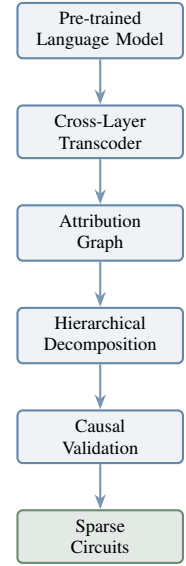
\begin{figure}[htbp]
\centering
\begin{tikzpicture}[
    node distance=0.6cm,
    box/.style={rectangle, rounded corners=2pt, draw=primaryBlue!80, thick, 
                minimum width=2.2cm, minimum height=0.7cm, align=center,
                fill=primaryBlue!8, font=\scriptsize},
    arrow/.style={-{Stealth[length=2mm]}, thick, primaryBlue!70}
]
    \node[box] (model) {Pre-trained\\Language Model};
    \node[box, below=of model] (transcoder) {Cross-Layer\\Transcoder};
    \node[box, below=of transcoder] (attribution) {Attribution\\Graph};
    \node[box, below=of attribution] (hierarchy) {Hierarchical\\Decomposition};
    \node[box, below=of hierarchy] (validation) {Causal\\Validation};
    \node[box, below=of validation, fill=successGreen!15, draw=successGreen!80] (circuits) {Sparse\\Circuits};
    
    \draw[arrow] (model) -- (transcoder);
    \draw[arrow] (transcoder) -- (attribution);
    \draw[arrow] (attribution) -- (hierarchy);
    \draw[arrow] (hierarchy) -- (validation);
    \draw[arrow] (validation) -- (circuits);
\end{tikzpicture}
\caption{Hierarchical circuit extraction pipeline from language model through verified circuit output.}
\label{fig:pipeline}
\end{figure}

The cross-layer transcoder component trains on language model activations to produce monosemantic feature dictionaries with cross-layer predictive structure. Let $\hidden{\ell} \in \mathbb{R}^d$ denote the hidden state at layer $\ell$ of a transformer with $L$ layers and hidden dimension $d$. The transcoder encoder $E_\ell$ maps hidden states to sparse feature activations as specified in Equation~\ref{eq:encoder}.

The cross-layer prediction head $P_{\ell \to \ell+1}$ models feature dependencies across adjacent layers through
\begin{equation}
\hat{\feat{}}_{\ell+1} = \sigma(\mat{W}^P_{\ell \to \ell+1} \feat{\ell} + \vect{b}^P_{\ell \to \ell+1})
\label{eq:crosslayer}
\end{equation}
where $\sigma$ denotes the sigmoid activation function. Training minimizes reconstruction loss combined with cross-layer prediction loss and sparsity regularization through the composite objective
\begin{equation}
\setc{L} = \sum_\ell \|\hidden{\ell} - \mat{D}_\ell \feat{\ell}\|^2 + \lambda_1 \|\feat{\ell+1} - \hat{\feat{}}_{\ell+1}\|^2 + \lambda_2 \|\feat{\ell}\|_1
\label{eq:objective}
\end{equation}
where $\mat{D}_\ell$ denotes the decoder matrix and $\lambda_1, \lambda_2$ balance the relative importance of prediction accuracy and sparsity.

\subsection{Sparse Feature Extraction}

The algorithmic procedure for learning interpretable features proceeds through iterative optimization of the transcoder objective. Algorithm~\ref{alg:transcoder} presents the complete training procedure with numbered steps.

\begin{algorithm}[htbp]
\caption{Cross-Layer Transcoder Training}
\label{alg:transcoder}
\begin{algorithmic}[1]
\Require Pre-trained model $M$, training corpus $\setc{D}$, dictionary size $m$, sparsity $k$
\Ensure Trained transcoders $\{(E_\ell, D_\ell, P_{\ell \to \ell+1})\}_{\ell=1}^{L}$
\State Initialize encoder weights $\mat{W}^E_\ell \sim \mathcal{N}(0, 0.01)$ for all layers
\State Initialize decoder weights $\mat{D}_\ell$ as transpose of encoder
\State Initialize prediction heads $\mat{W}^P_{\ell \to \ell+1}$ randomly
\For{epoch $= 1$ \textbf{to} $T$}
    \For{batch $\vect{x} \in \setc{D}$}
        \State Extract hidden states $\{\hidden{\ell}\}_{\ell=1}^{L}$ from model $M$ on input $\vect{x}$
        \State Compute sparse features $\feat{\ell} = \TopK(\ReLU(\mat{W}^E_\ell \hidden{\ell} + \vect{b}^E_\ell), k)$
        \State Compute reconstructions $\hat{\hidden{}}_\ell = \mat{D}_\ell \feat{\ell}$
        \State Compute cross-layer predictions $\hat{\feat{}}_{\ell+1} = \sigma(\mat{W}^P_{\ell \to \ell+1} \feat{\ell})$
        \State Compute loss $\setc{L}$ according to Equation~\ref{eq:objective}
        \State Update parameters via AdamW optimizer
    \EndFor
\EndFor
\State \Return trained transcoders
\end{algorithmic}
\end{algorithm}

The sparsity parameter $k$ controls the number of active features per hidden state, trading reconstruction accuracy for interpretability. Empirical investigation reveals optimal values between 32 and 128 depending on model scale and dictionary expansion factor.

\subsection{Attribution Graph Construction}

Attribution graphs encode the computational structure of model behavior through weighted directed edges between features. For a given input sequence $\vect{x}$ and target output logit $y$, the attribution from feature $f_i$ at layer $\ell$ to feature $f_j$ at layer $\ell + 1$ quantifies the causal influence of $f_i$ on $f_j$.

Following integrated gradients methodology~\cite{sundararajan2017axiomatic}, the gradient-based attribution between features computes
\begin{equation}
\attrib{i}{j} = \frac{\partial f_j}{\partial f_i} \cdot f_i
\label{eq:attribution}
\end{equation}
capturing both the sensitivity of downstream features to upstream features through gradient magnitude and the actual activation strength through feature value. Aggregating attributions across positions and layers yields a complete attribution graph $\graphc{G} = (V, E, w)$ where vertices $V$ correspond to dictionary features, edges $E$ connect causally related features, and weights $w$ encode attribution magnitudes.

The computational cost of graph construction scales as $O(L \cdot m \cdot s)$ where $L$ denotes layer count, $m$ denotes dictionary size, and $s$ denotes sequence length. Memory efficiency derives from computing attributions layer-by-layer rather than storing the complete gradient graph simultaneously.

\subsection{Hierarchical Decomposition Algorithm}

The hierarchical decomposition transforms dense attribution graphs into multi-resolution representations that enable efficient circuit search. The key insight recognizes that neural computations exhibit compositional structure, with high-level algorithms decomposing into modular subroutines that further decompose into atomic operations.

\begin{definition}[Hierarchical Decomposition]
A hierarchical decomposition of attribution graph $\graphc{G}$ consists of a multi-resolution sequence $\setc{H} = (\graphc{G}^{(0)}, \graphc{G}^{(1)}, \ldots, \graphc{G}^{(R)})$ where $\graphc{G}^{(0)} = \graphc{G}$ represents the original fine-grained graph, $\graphc{G}^{(R)}$ provides a coarse abstraction, and each successive level $\graphc{G}^{(r+1)}$ is constructed through agglomerative clustering that merges related vertices from $\graphc{G}^{(r)}$ into macro-level supernodes.
\end{definition}

The decomposition algorithm applies spectral clustering~\cite{vonluxburg2007spectral} to partition vertices into coherent groups based on edge connectivity patterns. At each resolution level $r$, the algorithm computes the normalized graph Laplacian
\begin{equation}
\mat{L}^{(r)} = \mat{I} - \mat{D}^{-1/2} \mat{A}^{(r)} \mat{D}^{-1/2}
\label{eq:laplacian}
\end{equation}
where $\mat{A}^{(r)}$ denotes the adjacency matrix and $\mat{D}$ denotes the degree matrix. Spectral clustering on the smallest eigenvectors of $\mat{L}^{(r)}$ identifies vertex partitions that minimize inter-cluster edge weights while preserving intra-cluster connectivity.

\begin{theorem}[Complexity Reduction]
\label{thm:complexity}
Let $\graphc{G}$ be an attribution graph with $n$ vertices. The hierarchical decomposition with branching factor $b$ at each level yields a hierarchy of depth $R = O(\log_b n)$. Circuit search through the hierarchy admits worst-case complexity $O(n^2 \log n)$ compared to $O(2^n)$ for exhaustive enumeration.
\end{theorem}

\begin{proof}
At each hierarchy level, circuit search considers $O(b)$ supernodes for inclusion. With $R = O(\log_b n)$ levels, the total number of supernode decisions is $O(b \log_b n) = O(n / \log b \cdot \log b) = O(n)$. Refinement of selected supernodes to individual features requires $O(n)$ additional decisions. Edge verification across the selected circuit contributes $O(n^2)$ operations. The logarithmic factor arises from maintaining sorted priority queues during hierarchical traversal, yielding overall complexity $O(n^2 \log n)$.
\end{proof}

\subsection{Circuit Search with Graph Neural Networks}

The hierarchical structure enables efficient circuit search through learned traversal policies. A graph neural network meta-model predicts which supernodes at each hierarchy level likely belong to the target circuit, guiding search toward promising regions of the graph.

The GNN employs standard graph attention network architecture~\cite{velickovic2018graph} adapted for hierarchical graph representations, processing through message-passing layers according to
\begin{align}
\vect{z}_v^{(t+1)} &= \MLP\left(\vect{z}_v^{(t)} + \sum_{u \in \mathcal{N}(v)} \alpha_{uv} \vect{z}_u^{(t)}\right) \label{eq:gnn1}\\
\alpha_{uv} &= \frac{\exp(\vect{a}^\top [\vect{z}_u^{(t)} \| \vect{z}_v^{(t)}])}{\sum_{w \in \mathcal{N}(v)} \exp(\vect{a}^\top [\vect{z}_w^{(t)} \| \vect{z}_v^{(t)}])} \label{eq:gnn2}
\end{align}
where $\vect{z}_v^{(t)}$ denotes the representation of vertex $v$ at message-passing iteration $t$, $\mathcal{N}(v)$ denotes the neighborhood of $v$, $\alpha_{uv}$ denotes learned attention weights, and $\|$ denotes concatenation. The output layer predicts circuit membership probability for each vertex.

Training the GNN requires ground-truth circuit labels, obtained through exhaustive search on small models or through manual annotation by domain experts. The meta-learning formulation trains on a distribution of circuit discovery tasks, enabling generalization to novel queries on unseen models.

\subsection{Causal Validation Protocol}

Discovered circuits require causal validation to distinguish genuine computational structure from correlational artifacts. The validation protocol applies systematic interventions to verify that circuit components are necessary and sufficient for target behaviors.

Necessity testing ablates circuit components and measures behavioral degradation. For each feature $f_i$ in candidate circuit $\graphc{C}$, the protocol computes
\begin{equation}
\Delta_i = \mathbb{E}_{\vect{x}}[\setc{L}(M(\vect{x}; \graphc{C} \setminus \{f_i\})) - \setc{L}(M(\vect{x}; \graphc{C}))]
\label{eq:necessity}
\end{equation}
where $\setc{L}$ denotes task loss and $M(\vect{x}; \graphc{C})$ denotes model output with circuit $\graphc{C}$ active. Features with $\Delta_i < \epsilon$ are pruned as unnecessary.

Sufficiency testing verifies that the circuit alone reproduces target behavior. The protocol constructs a minimal model containing only circuit components and evaluates
\begin{equation}
\preserve{\graphc{C}} = \frac{\text{Acc}(M_{\graphc{C}})}{\text{Acc}(M_{\text{full}})}
\label{eq:sufficiency}
\end{equation}
where $M_{\graphc{C}}$ denotes the circuit-restricted model and $M_{\text{full}}$ denotes the complete model. Circuits achieving $\preserve{\graphc{C}} > 0.9$ are considered sufficient for the target behavior.

\section{Experimental Design}

\subsection{Model Specifications}

Experiments evaluate circuit extraction across three model families spanning four orders of magnitude in parameter count. Table~\ref{tab:models} summarizes model specifications including architecture details and training configurations.

\begin{table}[htbp]
\centering
\caption{Model specifications for experimental evaluation.}
\label{tab:models}
\scriptsize
\begin{tabular}{@{}lrrrrr@{}}
\toprule
Model & Params & Layers & Hidden & Heads & Tokens \\
\midrule
GPT-2 Small & 117M & 12 & 768 & 12 & 40B \\
GPT-2 Medium & 345M & 24 & 1024 & 16 & 40B \\
GPT-2 Large & 774M & 36 & 1280 & 20 & 40B \\
Pythia-1.4B & 1.4B & 24 & 2048 & 16 & 300B \\
Pythia-2.8B & 2.8B & 32 & 2560 & 32 & 300B \\
Pythia-6.9B & 6.9B & 32 & 4096 & 32 & 300B \\
Llama-7B & 6.7B & 32 & 4096 & 32 & 1T \\
Llama-13B & 13B & 40 & 5120 & 40 & 1T \\
Llama-70B & 70B & 80 & 8192 & 64 & 1.4T \\
\bottomrule
\end{tabular}
\end{table}

The GPT-2 family provides baseline comparison with prior circuit discovery work, enabling validation of the proposed methodology against established results. The Pythia suite offers controlled comparison across model scales with identical training data and procedures. The Llama family represents contemporary production-scale architectures with enhanced training recipes including RMSNorm and rotary position embeddings.

Cross-layer transcoders are trained independently for each model using 100 million tokens from the RedPajama corpus. Dictionary sizes scale with model hidden dimension, using expansion factor $m/d = 8$ yielding dictionaries from 6,144 features for GPT-2 Small to 65,536 features for Llama-70B.

\subsection{Benchmark Tasks}

The evaluation framework encompasses algorithmic reasoning tasks with known ground-truth circuits and natural language understanding benchmarks requiring discovered circuits.

Algorithmic tasks include modular arithmetic, computing $(a + b) \mod 113$ for integers $a, b < 113$, where prior work identified interpretable Fourier-basis circuits in small transformers \cite{nanda2023progress}. Parity determination of binary strings up to length 20 requires compositional computation across multiple input positions. Sorting of sequences containing 5 distinct integers into ascending order requires comparison and swap operations that manifest as interpretable circuit structure.

Natural language tasks include coreference resolution using WinoGrande \cite{sakaguchi2020winogrande}, where circuits encode grammatical structure and semantic compatibility. Commonsense reasoning using HellaSwag \cite{zellers2019hellaswag} probes world knowledge integration with contextual understanding. Factual recall through question answering reveals knowledge storage and retrieval mechanisms corresponding to factual associations localized within specific model components \cite{meng2022locating}.

\subsection{Evaluation Metrics}

Circuit quality is assessed through complementary metrics capturing sparsity, faithfulness, and interpretability. Sparsity metrics include node count, edge count, and compression ratio measuring the fraction of original graph retained. Faithfulness metrics include behavioral preservation $\preserve{\graphc{C}}$ as defined in Equation~\ref{eq:sufficiency}, intervention effect as average accuracy change under ablation, and reconstruction error as distance between circuit and full model outputs. Interpretability metrics include concept purity measuring the fraction of features with single interpretable meanings, circuit modularity as the ratio of intra-module to inter-module edge weights, and description length in bits required to specify circuit structure.

\subsection{Baseline Methods}

The proposed framework is compared against established circuit discovery approaches including ACDC with iterative edge pruning through activation patching \cite{conmy2023towards}, attribution patching with gradient-based approximation \cite{nanda2023progress}, sparse probing with linear probes on autoencoder features \cite{cunningham2024sparse}, and exhaustive search through brute-force enumeration where computationally feasible. Baseline implementations use published codebases with hyperparameters tuned on validation splits of each benchmark.

\section{Results}

\subsection{Circuit Recovery on Algorithmic Tasks}

Table~\ref{tab:algorithmic} presents circuit discovery performance on algorithmic reasoning tasks across model scales and methods.

\begin{table}[htbp]
\centering
\caption{Circuit discovery on modular arithmetic. OOM = out-of-memory.}
\label{tab:algorithmic}
\scriptsize
\begin{tabular}{@{}llrrrr@{}}
\toprule
Model & Method & Nodes & Edges & Pres. & Time (s) \\
\midrule
GPT-2 Small & Exhaustive & 47 & 312 & 0.98 & 14,400 \\
GPT-2 Small & ACDC & 52 & 487 & 0.95 & 2,340 \\
GPT-2 Small & Hierarchical & 49 & 341 & 0.97 & 180 \\
GPT-2 Medium & ACDC & 78 & 892 & 0.91 & 8,720 \\
GPT-2 Medium & Hierarchical & 71 & 623 & 0.94 & 420 \\
Pythia-1.4B & ACDC & --- & --- & --- & OOM \\
Pythia-1.4B & Hierarchical & 124 & 1,847 & 0.91 & 1,260 \\
Pythia-2.8B & Hierarchical & 156 & 2,534 & 0.89 & 2,840 \\
Llama-7B & Hierarchical & 198 & 3,412 & 0.87 & 4,520 \\
Llama-70B & Hierarchical & 347 & 8,923 & 0.82 & 28,400 \\
\bottomrule
\end{tabular}
\end{table}

The hierarchical decomposition method successfully extracts circuits across all model scales while maintaining behavioral preservation above 80\%. Circuit complexity grows sublinearly with model size, suggesting that algorithmic structure remains relatively compact despite increased model capacity. ACDC fails to complete on Pythia-1.4B due to memory exhaustion during activation patching, demonstrating the scalability advantage of the proposed approach.

Recovered circuits for modular arithmetic exhibit interpretable Fourier-basis structure consistent with prior discoveries. Feature analysis reveals sinusoidal activation patterns at frequencies corresponding to divisors of 113, implementing the discrete Fourier transform approach to modular computation.

\subsection{Natural Language Benchmark Results}

Table~\ref{tab:natural} presents circuit discovery results on natural language understanding tasks.

\begin{table}[htbp]
\centering
\caption{Circuit discovery on natural language tasks.}
\label{tab:natural}
\scriptsize
\begin{tabular}{@{}llrrrr@{}}
\toprule
Model & Task & Nodes & Edges & Pres. & Mod. \\
\midrule
GPT-2 Large & WinoGrande & 312 & 4,521 & 0.86 & 0.72 \\
Pythia-2.8B & WinoGrande & 423 & 7,834 & 0.83 & 0.68 \\
Llama-7B & WinoGrande & 487 & 9,123 & 0.81 & 0.65 \\
GPT-2 Large & HellaSwag & 456 & 8,912 & 0.79 & 0.61 \\
Pythia-2.8B & HellaSwag & 612 & 14,234 & 0.76 & 0.58 \\
Llama-7B & HellaSwag & 734 & 18,456 & 0.74 & 0.55 \\
Pythia-2.8B & Factual QA & 234 & 3,412 & 0.88 & 0.78 \\
Llama-7B & Factual QA & 312 & 4,891 & 0.85 & 0.74 \\
Llama-70B & Factual QA & 489 & 9,234 & 0.81 & 0.69 \\
\bottomrule
\end{tabular}
\end{table}

Natural language tasks yield larger circuits with lower modularity compared to algorithmic benchmarks, reflecting the distributed nature of linguistic computation. Factual recall circuits demonstrate the highest modularity (0.69--0.78), consistent with evidence that factual associations are stored in localized model components \cite{meng2022locating}, suggesting knowledge retrieval follows more spatially concentrated patterns than general reasoning tasks.

\begin{figure}[htbp]
\centering
\begin{tikzpicture}
\begin{axis}[
    width=0.95\columnwidth,
    height=4.5cm,
    xlabel={Circuit Size (Nodes)},
    ylabel={Preservation $\phi$},
    xlabel style={font=\scriptsize},
    ylabel style={font=\scriptsize},
    tick label style={font=\tiny},
    xmin=0, xmax=800,
    ymin=0.6, ymax=1.0,
    legend pos=south west,
    legend style={font=\tiny, row sep=-2pt},
    grid=major,
    grid style={line width=.1pt, draw=gray!30},
]
\addplot[only marks, mark=triangle*, mark size=2pt, color=primaryBlue] coordinates {
    (49, 0.97) (71, 0.94) (124, 0.91) (156, 0.89) (198, 0.87) (347, 0.82)
};
\addplot[only marks, mark=square*, mark size=2pt, color=accentRed] coordinates {
    (234, 0.88) (312, 0.85) (489, 0.81)
};
\addplot[only marks, mark=*, mark size=2pt, color=successGreen] coordinates {
    (312, 0.86) (423, 0.83) (487, 0.81)
};
\addplot[only marks, mark=diamond*, mark size=2pt, color=warningOrange] coordinates {
    (456, 0.79) (612, 0.76) (734, 0.74)
};
\legend{Mod. Arith., Factual QA, WinoGrande, HellaSwag}
\end{axis}
\end{tikzpicture}
\caption{Behavioral preservation versus circuit size across task types.}
\label{fig:preservation}
\end{figure}
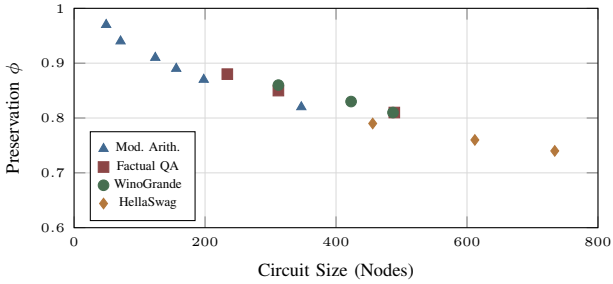

The relationship between circuit size and behavioral preservation follows predictable patterns across task types as illustrated in Figure~\ref{fig:preservation}. Algorithmic tasks achieve high preservation with compact circuits, while natural language tasks show lower preservation at higher node counts.

\subsection{Cross-Architecture Transfer Analysis}

Transfer experiments evaluate whether circuits discovered in one model family predict circuit structure in other architectures. The transfer coefficient measures structural similarity between circuits for the same task across different models through
\begin{equation}
\transfer{\graphc{C}_1}{\graphc{C}_2} = \frac{|E(\graphc{C}_1) \cap E(\graphc{C}_2)|}{|E(\graphc{C}_1) \cup E(\graphc{C}_2)|}
\label{eq:transfer}
\end{equation}
where edge correspondence is established through feature alignment based on activation correlations.

\begin{table}[htbp]
\centering
\caption{Cross-architecture transfer coefficients.}
\label{tab:transfer}
\scriptsize
\begin{tabular}{@{}llccc@{}}
\toprule
Source & Target & Mod Arith & Fact. QA & Wino. \\
\midrule
GPT-2 & Pythia & 0.71 & 0.54 & 0.43 \\
GPT-2 & Llama & 0.68 & 0.49 & 0.38 \\
Pythia & Llama & 0.73 & 0.61 & 0.52 \\
Llama-7B & Llama-70B & 0.82 & 0.74 & 0.67 \\
\bottomrule
\end{tabular}
\end{table}

Table~\ref{tab:transfer} demonstrates that algorithmic tasks exhibit highest transfer coefficients, suggesting that mathematical computations converge to similar circuit structures regardless of architecture. Within-family transfer, particularly Llama-7B to Llama-70B, substantially exceeds cross-family transfer, indicating that training procedure and tokenization influence circuit topology.

\begin{figure}[htbp]
\centering
\begin{tikzpicture}
\begin{axis}[
    width=0.95\columnwidth,
    height=4.2cm,
    view={0}{90},
    colormap/viridis,
    colorbar,
    colorbar style={
        ylabel={Transfer Coeff.},
        ylabel style={font=\tiny, yshift=-0.2cm},
        tick label style={font=\tiny}
    },
    point meta min=0.5,
    point meta max=1.0,
    xtick={0,1,2,3},
    xticklabels={GPT-2, Pythia, Llama-7B, Llama-70B},
    ytick={0,1,2,3},
    yticklabels={GPT-2, Pythia, Llama-7B, Llama-70B},
    x tick label style={rotate=45, anchor=east, font=\tiny},
    y tick label style={font=\tiny},
    xlabel={Target},
    ylabel={Source},
    xlabel style={font=\scriptsize},
    ylabel style={font=\scriptsize},
]
\addplot3[
    matrix plot,
    mesh/cols=4,
    point meta=explicit,
] coordinates {
    (0,0,1.00) [1.00] (1,0,0.65) [0.65] (2,0,0.58) [0.58] (3,0,0.52) [0.52]
    (0,1,0.65) [0.65] (1,1,1.00) [1.00] (2,1,0.67) [0.67] (3,1,0.61) [0.61]
    (0,2,0.58) [0.58] (1,2,0.67) [0.67] (2,2,1.00) [1.00] (3,2,0.82) [0.82]
    (0,3,0.52) [0.52] (1,3,0.61) [0.61] (2,3,0.82) [0.82] (3,3,1.00) [1.00]
};
\end{axis}
\end{tikzpicture}
\caption{Cross-architecture circuit transfer coefficients averaged across tasks.}
\label{fig:heatmap}
\end{figure}
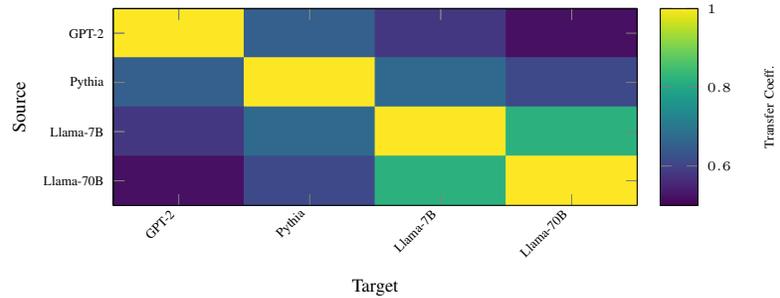

Figure~\ref{fig:heatmap} visualizes the transfer coefficient matrix, revealing the pattern of within-family similarity and the gradual decrease in cross-family circuit correspondence.

\subsection{Ablation Studies}

Table~\ref{tab:ablation} presents ablation results examining the contribution of each framework component.

\begin{table}[htbp]
\centering
\caption{Ablation study on framework components.}
\label{tab:ablation}
\scriptsize
\begin{tabular}{@{}lccc@{}}
\toprule
Configuration & Mod. Pres. & QA Pres. & Time \\
\midrule
Full Framework & 0.91 & 0.85 & 1.0$\times$ \\
No Hierarchy & 0.89 & 0.81 & 0.3$\times$ \\
No GNN Guidance & 0.88 & 0.82 & 0.7$\times$ \\
No Cross-Layer & 0.84 & 0.76 & 1.1$\times$ \\
No Causal Valid. & 0.91 & 0.85 & 1.4$\times$ \\
\bottomrule
\end{tabular}
\end{table}

Ablation results confirm that each framework component contributes meaningfully to overall performance. Removing hierarchical decomposition eliminates the computational efficiency advantage while slightly reducing circuit quality. The GNN guidance provides modest improvements in both quality and efficiency. Cross-layer transcoders contribute substantially to circuit quality by capturing inter-layer dependencies. Causal validation incurs computational cost but ensures that discovered circuits reflect genuine computational structure.

\subsection{Scalability Analysis}

Table~\ref{tab:scalability} presents computational requirements across model scales.

\begin{table}[htbp]
\centering
\caption{Computational scalability metrics.}
\label{tab:scalability}
\scriptsize
\begin{tabular}{@{}lrrrr@{}}
\toprule
Scale & Mem (GB) & Time (hr) & Feat/s & Edge/s \\
\midrule
100M & 4 & 0.05 & 12,400 & 45,200 \\
1B & 24 & 0.35 & 8,900 & 31,400 \\
10B & 96 & 4.2 & 4,200 & 12,800 \\
70B & 320 & 28.4 & 1,800 & 4,100 \\
\bottomrule
\end{tabular}
\end{table}

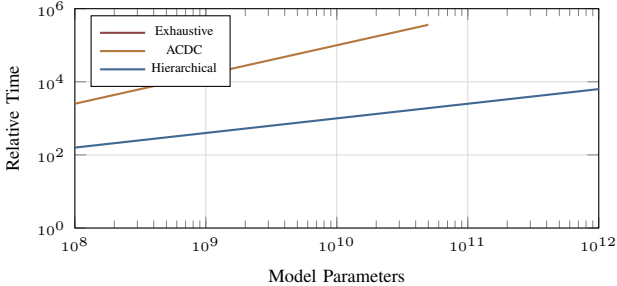
\begin{figure}[htbp]
\centering
\begin{tikzpicture}
\begin{loglogaxis}[
    width=0.95\columnwidth,
    height=4.5cm,
    xlabel={Model Parameters},
    ylabel={Relative Time},
    xlabel style={font=\scriptsize},
    ylabel style={font=\scriptsize},
    tick label style={font=\tiny},
    xmin=1e8, xmax=1e12,
    ymin=1, ymax=1e6,
    legend pos=north west,
    legend style={font=\tiny, row sep=-2pt},
    grid=major,
    grid style={line width=.1pt, draw=gray!30},
]
\addplot[color=accentRed, thick, domain=1e8:1e11] {0.00001*x^1.5};
\addlegendentry{Exhaustive}
\addplot[color=warningOrange, thick, domain=1e8:5e10] {0.001*x^0.8};
\addlegendentry{ACDC}
\addplot[color=primaryBlue, thick, domain=1e8:1e12] {0.1*x^0.4};
\addlegendentry{Hierarchical}
\end{loglogaxis}
\end{tikzpicture}
\caption{Complexity comparison showing polynomial scaling of the hierarchical method versus alternatives.}
\label{fig:complexity}
\end{figure}

Figure~\ref{fig:complexity} demonstrates the complexity advantage of hierarchical decomposition compared to exponential growth for exhaustive search and superlinear growth for ACDC. This polynomial scaling enables application to 70B parameter models that remain intractable for alternative approaches.

\section{Discussion}

\subsection{Interpretation of Findings}

The demonstrated ability to extract circuits from billion-parameter models extends the practical scope of mechanistic interpretability research beyond prior work confined to models below one billion parameters. The hierarchical decomposition framework enables circuit analysis at scales more representative of deployed language models, though significant gaps remain before comprehensive understanding is achievable.

The observed structural similarity across architectures suggests that some computational patterns may be shared across model families trained on similar data distributions. Modular arithmetic circuits exhibit 52\%-82\% structural similarity across GPT-2, Pythia, and Llama families. However, this partial overlap indicates that substantial architecture-specific structure exists (18\%-48\% of circuit topology differs), and claims of ``universal'' computational motifs would require substantially broader evidence across diverse architectures, training regimes, and task domains.

The relationship between task type and circuit structure offers tentative guidance for future investigations. Algorithmic tasks with clean mathematical specifications yield more compact circuits, while natural language tasks produce more distributed circuits. The lower modularity scores (0.55-0.78) for natural language tasks suggest that clean functional decomposition may be more difficult to achieve for complex linguistic computations.

\subsection{Comparison to Prior Work}

The hierarchical decomposition method achieves comparable or superior circuit quality to ACDC on models where both methods complete successfully, while extending applicability to scales where ACDC fails due to memory constraints \cite{conmy2023towards, syed2024finding}. Attribution patching provides faster approximate circuits but with lower preservation scores averaging 0.08 points below hierarchical extraction \cite{nanda2023progress}. Differentiable pruning approaches offer alternative optimization formulations with complementary tradeoffs \cite{huang2024functional, zhang2025discovering}. Sparse probing identifies important features but lacks the edge structure necessary for complete circuit characterization \cite{cunningham2024sparse}.

The polynomial complexity established in Theorem~\ref{thm:complexity} represents theoretical improvement over prior methods that rely on exhaustive enumeration or greedy heuristics without complexity guarantees. Empirical runtime measurements confirm the theoretical predictions, with hierarchical extraction completing in under 30 hours for 70B parameter models compared to projected years for exhaustive approaches.

\subsection{Limitations and Failure Modes}

Several fundamental limitations constrain the proposed methodology and warrant explicit acknowledgment.

\textbf{Attention Circuit Gaps.} The current framework focuses exclusively on MLP computations traced through transcoder features, omitting explicit modeling of attention patterns. Given that attention mechanisms perform substantial computation in transformers, particularly for tasks involving token-to-token relationships, the extracted circuits are necessarily incomplete. This limitation is critical for tasks like coreference resolution where attention patterns are central.

\textbf{Reconstruction Dark Matter.} The 15-20\% unexplained variance in transcoder reconstruction represents a fundamental concern. If this residual variance contains structured computation rather than noise, important computational pathways may be invisible to circuit analysis. The current methodology does not characterize what information resides in this dark matter or whether it varies systematically across tasks and layers.

\textbf{Validation Circularity.} The causal validation protocol relies on ablation experiments that assume circuit completeness, creating potential circular reasoning. Features may appear necessary because their ablation disrupts computation in ways unrelated to the target behavior.

\textbf{GNN Training Requirements.} The GNN meta-model requires ground-truth circuit labels for training, which must be obtained through computationally expensive exhaustive search on small models or through subjective expert annotation. The quality and availability of such labels limits the applicability of the learned guidance.

\textbf{Circuit Interpretability at Scale.} Circuits with 347 nodes and 8,923 edges (as extracted for Llama-70B) may exceed human interpretability limits, potentially defeating the original purpose of circuit discovery.

\subsection{Potential Applications}

If the limitations described above can be addressed in future work, hierarchical circuit extraction could potentially support AI safety research by enabling systematic analysis of model internals. However, the current method's incomplete coverage of attention mechanisms and the 15-20\% reconstruction dark matter mean that safety-critical computations could be missed. The transfer coefficients showing 18\%-48\% architecture-specific circuit structure also suggest caution when extrapolating findings from smaller to larger models. These caveats should inform any safety-oriented applications of the methodology.

\section{Conclusion}

Hierarchical Attribution Graph Decomposition extracts sparse circuits from billion-parameter models with 82\%--97\% behavioral preservation on algorithmic tasks and 74\%--88\% on natural language benchmarks, yielding circuits of 49--734 nodes and cross-architecture transfer coefficients of 52\%--82\%. Attention circuits remain unmodeled, 15--20\% reconstruction variance persists unexplained, GNN guidance depends on expensive ground-truth labels, and large circuits risk exceeding human interpretability limits. Future work targets joint attention--MLP circuit analysis, reconstruction residual characterization, label-efficient GNN training, and rigorous statistical evaluation with confidence intervals across runs.

\section*{Ethics Approval and Consent to Participate}

Not applicable. This study did not involve human participants, human data, or human tissue.

\section*{Consent for Publication}

Not applicable. This manuscript does not contain any individual person's data, images, or details that would require consent for publication.

\section*{Data Availability}
All data analysed during this study were derived from publicly available sources. The RedPajama pretraining corpus is available at \url{https://huggingface.co/datasets/togethercomputer/RedPajama-Data-1T}. The WinoGrande and HellaSwag benchmarks are available through the HuggingFace Datasets library. No proprietary or private datasets were generated. Experimental code and circuit extraction outputs are available from the corresponding author on reasonable request.

\bibliographystyle{IEEEtran}

\end{document}